\documentclass[sigconf, nonacm]{acmart}

\newcommand\vldbdoi{10.14778/3547305.3547316}
\newcommand\vldbpages{2111 - 2120}
\newcommand\vldbvolume{15}
\newcommand\vldbissue{10}
\newcommand\vldbyear{2022}
\newcommand\vldbauthors{\authors}
\newcommand\vldbtitle{\shorttitle}
\newcommand\vldbavailabilityurl{https://github.com/ccchengff/FDL/tree/main/playground/celu_vfl} 
\newcommand\vldbpagestyle{empty}

\usepackage{amsmath}
\usepackage{amsthm}
\usepackage{amssymb}
\usepackage{mathtools}
\usepackage{bm}
\usepackage{bbm}
\usepackage{booktabs}
\usepackage{xspace}
\usepackage{color,soul}
\usepackage{multirow}
\usepackage{multicol}
\usepackage{caption}
\usepackage{enumitem}

\usepackage{microtype}
\usepackage{graphicx}
\usepackage{subfigure}
\usepackage{booktabs}
\usepackage{hyperref}
\usepackage{nicefrac}
\usepackage[normalem]{ulem}
\usepackage{tikz}
\usetikzlibrary{trees}
\usepackage{float}
\usepackage{subfigure}
\usepackage[ruled,noend,linesnumbered,vlined]{algorithm2e}

\newcommand{\system}{{\texttt{CELU-VFL}}\xspace}

\definecolor{myred}{rgb}{1.0,0.7,0.8}
\definecolor{mygreen}{RGB}{0,166,0}
\definecolor{lightgreen}{rgb}{0.56, 0.93, 0.56}
\definecolor{myorange}{RGB}{252,107,4}
\definecolor{darkgreen}{RGB}{0,153,102}
\definecolor{lightblue}{rgb}{0.53, 0.81, 0.92}
\definecolor{lightgray}{gray}{0.9}

\newcommand{\mysubsubsection}[1]{{\textbf{\textit{#1}.\xspace}}}

\DeclareMathOperator*{\argmin}{arg\,min}
\newcommand{\eps}{\varepsilon}

\newcommand{\host}{\textit{Party {\small $\mathcal{A}$}}\xspace}
\newcommand{\hosti}{\textit{Party {\small ${\mathcal{A}_i}$}}\xspace}
\newcommand{\guest}{\textit{Party {\small $\mathcal{B}$}}\xspace}

\newcommand{\anyparty}{\textit{Party} {\small $\mathcal{P}$}\xspace}

\newcommand{\xa}{X_A}

\newcommand{\xb}{X_B}
\newcommand{\xp}{X_P}
\newcommand{\za}{Z_A}
\newcommand{\zai}{Z_{A_i}}
\newcommand{\zb}{Z_B}
\newcommand{\zp}{Z_P}
\newcommand{\wba}{\theta_{Bot\_A}}
\newcommand{\wbai}{\theta_{Bot\_{A_i}}}
\newcommand{\wbb}{\theta_{Bot\_B}}
\newcommand{\wbp}{\theta_{Bot\_P}}
\newcommand{\wt}{\theta_{Top}}

\newtheorem{theorem}{Theorem}

\newtheorem{lemma}[theorem]{Lemma}

\newtheorem{claim}{Claim}

\newtheorem*{assumption*}{\assumptionnumber}
\providecommand{\assumptionnumber}{}
\makeatletter
\newenvironment{assumption}[1]
 {%
  \renewcommand{\assumptionnumber}{Assumption #1}%
  \begin{assumption*}%
  \protected@edef\@currentlabel{#1}%
 }
 {%
  \end{assumption*}
 }
\makeatother

\newenvironment{subproof}[1][\proofname]{%
  \begin{proof}[#1]%
}{%
  \end{proof}%
}

\begin{document}

%%
%% The "title" command has an optional parameter,
%% allowing the author to define a "short title" to be used in page headers.
\title{Towards Communication-efficient Vertical Federated Learning Training via Cache-enabled Local Updates}

%%
%% The "author" command and its associated commands are used to define
%% the authors and their affiliations.
\author{Fangcheng Fu}
\authornote{School of Computer Science \& Key Lab of High Confidence Software Technologies (MOE), Peking University}
\affiliation{
\institution{Peking University}
}
\email{ccchengff@pku.edu.cn}

\author{Xupeng Miao}
\authornotemark[1]
\affiliation{
\institution{Peking University}
}
\email{xupeng.miao@pku.edu.cn}

\author{Jiawei Jiang}
\authornote{School of Computer Science, Wuhan University}
\affiliation{
\institution{Wuhan University}
}
\email{jiawei.jiang@whu.edu.cn}

\author{Huanran Xue}
\affiliation{
\institution{Tencent Inc.}
}
\email{huanranxue@tencent.com}

\author{Bin Cui}
\authornotemark[1]
\authornote{Institute of Computational Social Science, Peking University (Qingdao), China}
\affiliation{
\institution{Peking University}
}
\email{bin.cui@pku.edu.cn}

%%
%% The abstract is a short summary of the work to be presented in the
%% article.
\begin{abstract}
Vertical federated learning (VFL) is an emerging paradigm that 
allows different parties (e.g., organizations or enterprises) to 
collaboratively build machine learning models with privacy protection. 
In the training phase, VFL only exchanges the intermediate statistics, 
i.e., forward activations and backward derivatives, 
across parties to compute model gradients. 
Nevertheless, due to its geo-distributed nature, 
VFL training usually suffers from the low WAN bandwidth. 

In this paper, we introduce \system, a novel and efficient VFL training framework 
that exploits the local update technique to reduce the cross-party communication rounds. 
\system caches the stale statistics and reuses them to estimate model gradients 
without exchanging the ad hoc statistics. 
Significant techniques are proposed to improve the convergence performance. 
First, to handle the stochastic variance problem, 
we propose a uniform sampling strategy to 
fairly choose the stale statistics for local updates. 
Second, to harness the errors brought by the staleness, 
we devise an instance weighting mechanism that measures the reliability 
of the estimated gradients. 
Theoretical analysis proves that \system achieves a similar sub-linear convergence rate 
as vanilla VFL training but requires much fewer communication rounds. 
Empirical results on both public and real-world workloads 
validate that \system can be up to six times faster than the existing works. 
\end{abstract}

%%
%% This command processes the author and affiliation and title
%% information and builds the first part of the formatted document.
\maketitle

%%% do not modify the following VLDB block %%
%%% VLDB block start %%%
\pagestyle{\vldbpagestyle}
\begingroup\small\noindent\raggedright\textbf{PVLDB Reference Format:}\\
\vldbauthors. \vldbtitle. PVLDB, \vldbvolume(\vldbissue): \vldbpages, \vldbyear.\\
\href{https://doi.org/\vldbdoi}{doi:\vldbdoi}
\endgroup
\begingroup
\renewcommand\thefootnote{}\footnote{\noindent
This work is licensed under the Creative Commons BY-NC-ND 4.0 International License. Visit \url{https://creativecommons.org/licenses/by-nc-nd/4.0/} to view a copy of this license. For any use beyond those covered by this license, obtain permission by emailing \href{mailto:info@vldb.org}{info@vldb.org}. Copyright is held by the owner/author(s). Publication rights licensed to the VLDB Endowment. \\
\raggedright Proceedings of the VLDB Endowment, Vol. \vldbvolume, No. \vldbissue\ %
ISSN 2150-8097. \\
\href{https://doi.org/\vldbdoi}{doi:\vldbdoi} \\
}\addtocounter{footnote}{-1}\endgroup
%%% VLDB block end %%%

%%% do not modify the following VLDB block %%
%%% VLDB block start %%%
\ifdefempty{\vldbavailabilityurl}{}{
\vspace{.3cm}
\begingroup\small\noindent\raggedright\textbf{PVLDB Artifact Availability:}\\
The source code, data, and/or other artifacts have been made available at \url{\vldbavailabilityurl}.
\endgroup
}
%%% VLDB block end %%%

\section{Introduction}
\label{sec:intro}

The past few years have witnessed the explosion of 
federated learning (FL) research and applications
~\cite{fl_concepts,fl_challenge,konevcny2016federated_opt,konevcny2016federated_learn}, 
where two or more parties (participants) 
collaboratively train a machine learning (ML) model 
using their own private data. 
In particular, vertical FL (VFL) 
illustrates such a scenario where different parties 
hold disjoint features but overlap on a set of instances. 
As shown in Figure~\ref{fig:vfl}, 
the overlapping instances are regarded as a virtual dataset 
that is vertically partitioned. 
There are two types of parties: 
(i) one or more \host's that merely provide features; 
and (ii) one \guest that holds both features and 
the ground truth labels for the learning tasks. 
The goal of VFL is to fit the labels by 
the features of all parties, 
whilst keeping the private data (including both features and labels) 
of each party secret from the others. 
VFL is suitable for a variety of cross-enterprise collaborations. 
For instance, \host can be an Internet company that 
is able to generate rich and precise user profiling, 
whilst \guest can be an e-commercial company 
that hopes to build a better recommender system\footnote{
In Practice, two or more \host's would be possible. 
However, such a two-party setting is 
the most widely seen circumstance 
for cross-enterprise collaborations. 
Furthermore, our work can be generalized to 
two or more \host's easily. 
Therefore, we only consider one \host in this work 
for simplicity.}.
As a result, VFL is gaining more popularity 
in both academia and industry
~\cite{vf2boost,vfl_lr,secureboost,privacy_vfl_tree}.

\begin{figure}[!t]
\centering
\includegraphics[width=2.8in]{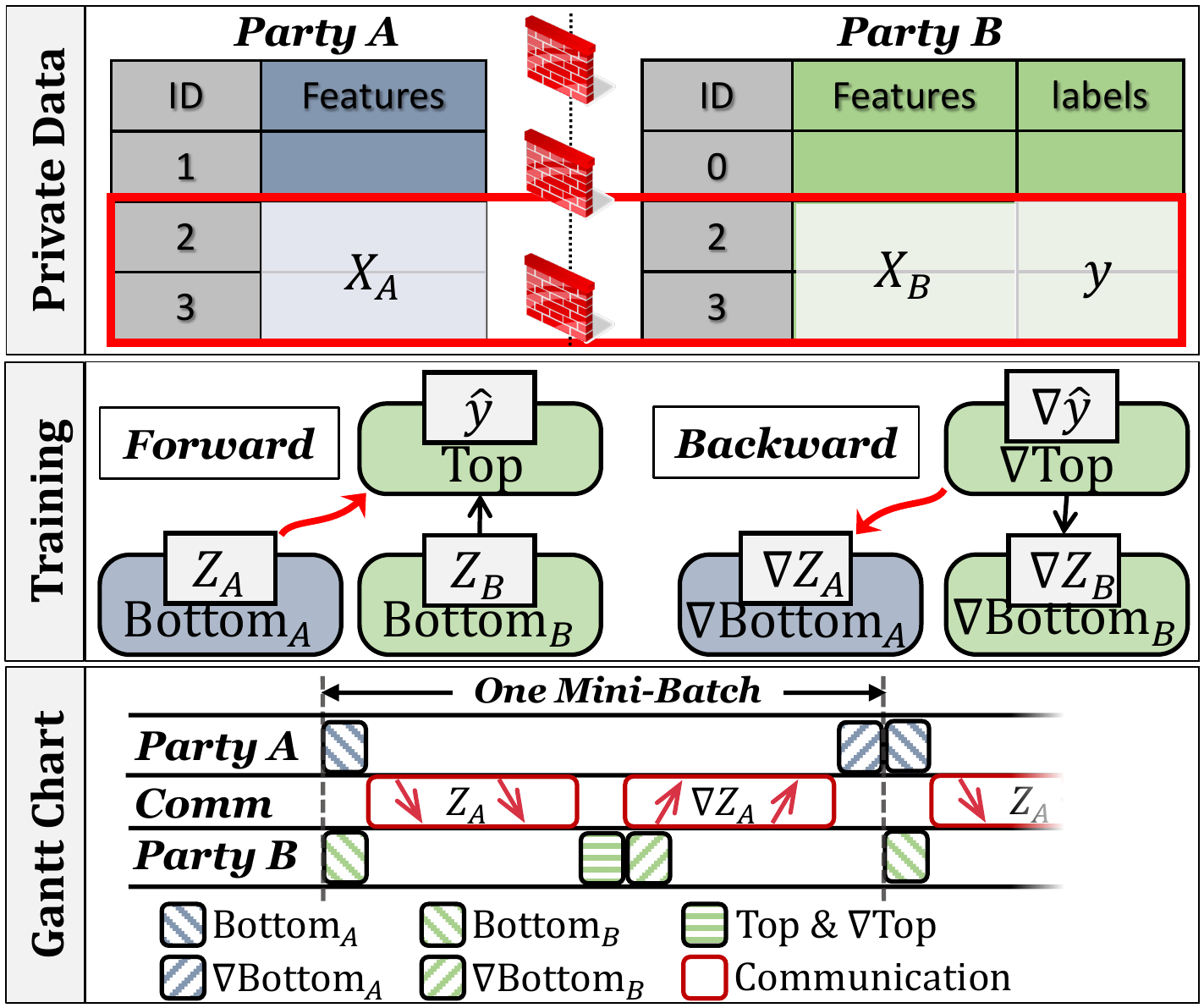}
\caption{\small{An illustration of VFL. The cross-party communication (indicated by red arrows) dominates the time cost of training.}}
\label{fig:vfl}
\end{figure}

{\textbf{(Communication Bottleneck of VFL)}}
Denote {\small $y$} as the labels owned by \guest, 
and {\small$\xa, \xb$} as the features 
of \host and \guest, respectively, 
where the $k$-th rows {\small$X_{A,k}, X_{B,k}$} 
together form the features of the {$k$}-th instances. 
As shown in Figure~\ref{fig:vfl}, 
VFL models typically follow a bottom-to-top architecture. 
First, each \anyparty is associated with a bottom model 
that works as a feature extractor. 
It takes {\small$\xp$} as input 
and outputs {\small$\zp = \texttt{Bottom}_P(\xp; \wbp)$}, 
where {\small$\wbp$} are the bottom model weights. 
Second, \guest is further associated with a top model 
parameterized by {\small$\wt$} that makes predictions via 
{\small$\hat{y} = \texttt{Top}(\za,\zb; \wt)$}. 
The goal of VFL training is to find the model weights 
minimizing the empirical risk: 
\begin{equation}
\label{eq:vfl}
\small
% \centerline{$% disable math mode
	\argmin_{\left\{\wba, \wbb, \wt\right\}} \; L \coloneqq 
	\sum_k l(y_k, \hat{y}_k), 
% $}
\end{equation}
where {\small$l$} is the loss function (e.g., the logistic loss).

The most common avenue to solve such a 
supervised ML problem 
is to adopt the mini-batch stochastic gradient descent 
(SGD) algorithm or its variants. 
For each step, a batch of instances is sampled 
to compute the model gradients 
{\small$\nabla\theta = \partial L / \partial\theta$} 
for model updates. 
However, since \host must not get access to the labels 
or the top model due to the privacy consideration, 
it cannot accomplish the backward propagation process alone. 
To overcome this problem, existing VFL frameworks  
adopt a two-phase propagation routine --- 
the gradients for {\small$\wba$} are calculated via 
{\small$\nabla \wba = \nabla \za \frac{\partial \za}{\partial \wba}$}, 
where {\small$\nabla \za = \partial L / \partial \za$} 
are provided by \guest. 

As shown in the Gantt chart of Figure~\ref{fig:vfl}, 
in each step, two kinds of intermediate data are transmitted, 
i.e., \underline{forward activations} {\small$\za$} 
and \underline{backward derivatives} {\small$\nabla\za$}. 
In practice, the communication cost is expensive due to
the limited network bandwidth. 
Since the data centers of different parties 
are usually physically distributed, 
VFL tasks have to run in a geo-distributed manner 
where Wide Area Network (WAN) bandwidth is very low. 
For instance, as reported by~\citet{gaia_nsdi} and verified by our empirical benchmark, 
the bandwidth between different data centers 
are usually smaller than 300 Mbps.
Worse still, since most companies 
have network restrictions 
to avoid external attacks, 
the servers that carry out the VFL tasks 
are forbidden from connecting to WAN directly. 
Instead, messages are proxied by some gateway machines, 
leading to even slower communication. 
Eventually, in real-world VFL applications, 
we observe that over 90\% of 
the total training time is spent on communication, 
leading to extremely low efficiency 
and a huge waste of computational resources. 
Consequently, how to tackle the communication bottleneck 
in VFL is a timely and important problem.

\begin{algorithm}[!t]
\caption{{\small{VFL training with local updates (FedBCD~\cite{fedbcd}) for \host. Line 4-5 can be executed in background to overlap the communication and local updates.}}}
\small
\label{alg:vfl_with_local_updates}
%\DontPrintSemicolon

\ForEach{\textup{mini-batch} $i$}
{
	Compute and send {\small$\za^{(i)}$}; 
	Recv {\small$\nabla\za^{(i)}$}\;
	Update with {\small$\nabla \wba = \nabla \za^{(i)} \frac{\partial \za^{(i)}}{\partial \wba}$}\;
	\For(\tcp*[h]{local updates}){$j \gets 1 \textup{ until } R$}
	{
		Compute {\small$\za^{(i, j)}$};
		Update with {\small$\widetilde{\nabla} \wba = \nabla \za^{(i)} \frac{\partial \za^{(i, j)}}{\partial \wba}$}\;
	}
}
\end{algorithm}

\textbf{(Motivation)}
To address the communication bottleneck, 
an intuitive idea is to amortize the network cost by 
performing multiple local updates between two communication rounds.
In fact, such a local update technique (a.k.a. local SGD) 
has been widely studied in conventional (non-federated) distributed ML and horizontal FL
~\cite{local_sgd_1,local_sgd_2,konevcny2016federated_opt,konevcny2016federated_learn}. 
However, applying this technique in VFL is non-trivial 
since each party cannot accomplish 
the forward or backward propagation individually 
due to the lack of {\small$\za$} or {\small$\nabla \za$}. 
Thus, approximating {\small$\za$} or {\small$\nabla \za$} with stale statistics is needed. 
Specifically, FedBCD~\cite{fedbcd} proposes to execute {\small$R$} updates for each mini-batch, 
which is shown in Algorithm~\ref{alg:vfl_with_local_updates}.
During the {$j$}-th update of the {$i$}-th mini-batch, 
instead of exchanging the ad hoc values, 
i.e., {\small$\langle \za^{(i,j)},\nabla\za^{(i,j)} \rangle$}, 
\host treats {\small$\nabla\za^{(i)}$} as an approximation and 
estimates the model gradients as 
{\small$\widetilde{\nabla}\wba = \nabla\za^{(i)} \frac{\partial\za^{(i,j)}}{\partial\wba}$}; 
symmetrically, \guest feeds the stale values {\small$\za^{(i)}$} 
and the ad hoc values {\small$\zb^{(i,j)}$} to the top model.

Nevertheless, there are two drawbacks of local updates in VFL:
(1. variance problem) 
using the same stale mini-batch for multiple consecutive steps 
would enlarge the stochastic variance in model gradients 
(similar phenomenon can be observed in SGD algorithms without data shuffling~\cite{shuffle1,shuffle2,corgipile});
(2. approximation errors) 
because approximation via stale statistics inevitably introduces 
errors to the model gradients, 
it would slow down the convergence or even lead to divergence 
if we directly update the models with the approximated gradients. 
According to the above analysis, there exists a discrepancy 
in the state of VFL with local updates --- 
doing more local steps brings a better improvement on the \textit{system efficiency}, 
but can easily harm the \textit{statistical efficiency}.

\textbf{(Summary of Contributions)}
Motivated by this, our work proposes a novel VFL training framework that
improves both system efficiency and statistical efficiency 
at the same time. 
To summarize:
\begin{itemize}[leftmargin=*]
\item 
We develop \system, 
an efficient VFL training framework that 
exploits a cache-enabled local update technique 
for better system efficiency. 
To reduce the communication cost, 
\system introduces an abstraction of workset table that 
caches the forward activations and backward derivatives 
to enable the local update technique. 
\system further overlaps the reads and writes of the workset table 
triggered by local computation and cross-party communication 
to improve resource utilization.

\item 
To improve the statistical efficiency, 
\system addresses the two drawbacks of local updates in VFL  
based on the workset table abstraction. 
(1) First, to tackle the variance problem, 
we introduce a uniform local sampling strategy --- 
instead of repetitively using the same mini-batch, 
the cached statistics are fairly sampled from the workset table 
to execute the local updates. 
(2) Second, to mitigate the impact of approximation errors, 
we design a staleness-aware instance weighting mechanism. 
The essential idea is to measure the staleness using cosine similarities 
and control the contribution of each instance according to its staleness. 

\item 
Furthermore, we theoretically analyze the convergence property of \system, 
which shows that our work achieves a similar sub-linear convergence rate 
as vanilla VFL training but requires much fewer communication rounds. 
The analysis also interprets why our uniform local sampling strategy 
and instance weighting mechanism improve the statistical efficiency. 

\item 
We empirically evaluate the performance of \system through extensive experiments. 
First, we investigate the statistical efficiency of \system along three dimensions, 
i.e., the number of local updates, the strategy of local sampling, 
and the weighting mechanism of staleness. 
Second, end-to-end evaluation on both public datasets and real-world workloads 
demonstrates that \system is able to outperform the existing works 
by 2.65-6.27$\times$. 

\end{itemize}

\section{Preliminaries}
\label{sec:bg}

\subsection{Vertical Federated Learning}
Owing to the establishment of lawful regulation 
on privacy protection~\cite{gdpr,ccpa,cdpa}, 
how to train ML models with privacy preservation 
is becoming a hot topic. 
Notably, VFL has become a prevailing paradigm 
that unites the features from different parties 
for collaborative applications
~\cite{vfl_lr,blindfl,vf2boost,secureboost,privacy_vfl_tree,fdml,privacy_grad,vfl_label_protect,fl_concepts}. 

\mysubsubsection{Privacy Preservation}
Unlike conventional distributed ML, 
raw data are not collected together in VFL, 
and model weights or updates are not transferred 
between parties, either. 
Instead, only the intermediate statistics, 
i.e., forward activations and backward derivatives, 
are exchanged between parties 
to protect the data and models. 
Many existing works have theoretically analyzed the 
privacy guarantees of VFL under the honest-but-curious assumption
~\cite{async_vfl_2,async_vfl_3,vafl,vfl_label_protect}. 

\mysubsubsection{Data Management}
Under the VFL setting, the data in both parties are assumed to be vertically aligned. 
In other words, the overlapping instances 
(e.g., instances 2 and 3 in Figure~\ref{fig:vfl}) 
have been extracted and aligned in the same ordering before training. 
This can be done by the private set intersection (PSI) technique~\cite{psi1,psi2,psi3}. 
Then, both parties can sample the mini-batches using the same random seed 
to ensure that each mini-batch is also aligned.

\mysubsubsection{The Cross-party Communication Bottleneck}
As introduced in Section~\ref{sec:intro}, since different parties 
are usually geo-distributed and the WAN bandwidth is limited in practice, 
the cross-party communication cost inevitably becomes the major bottleneck in VFL. 
Take a real case as an example, where the output dimensionality of {\small$\za$} is 256 
and the batch size is 4096. 
The size of {\small$\za$} or {\small$\nabla\za$} 
will be 4MB using single-precision floating-point numbers. 
Eventually, each communication round (including two transmissions) 
will take 213 milliseconds given a 300Mbps network bandwidth. 
In contrast, the computation time is much shorter thanks to 
the explosive growth of modern accelerators (e.g., GPUs) in the past decade. 
As a result, how to reduce the communication cost and improve the utilization 
of computational resources 
has become an urgent issue in VFL.

\subsection{Local Update}
To address the communication deficiency in distributed training, 
umpteenth works are developed to reduce the communication rounds 
by conducting local updates
~\cite{local_sgd_1,local_sgd_2,konevcny2016federated_opt,konevcny2016federated_learn,fedprox}. 
However, they mainly focus on distributed ML or horizontal FL (HFL), 
where the datasets are horizontally partitioned 
among different machines (training workers or clients). 
Under such a setting, the label and features for the same instance are co-located 
so that each machine can accomplish 
the forward and backward propagation alone. 

However, in VFL, parties need to fetch 
the forward activations or backward derivatives 
in each step. 
To apply the local update technique in VFL, 
we have to approximate the intermediate statistics 
rather than fetching from the other party. 
Since ML models are trained iteratively, there is a common sense that 
the models do not change significantly in a few iterations 
and are robust to certain noises. 
Thus, \citet{fedbcd} propose to approximate the ad hoc statistics 
with the stale ones. 
Specifically, they reuse the statistics of the last batch for multiple iterations 
and update the models with the estimated gradients directly. 
However, as discussed in Section~\ref{sec:intro}, 
the convergence would be harmed due to 
the enlarged stochastic variance and the approximation errors 
in local updates.

\begin{figure*}[!t]
\centering
\includegraphics[width=6in]{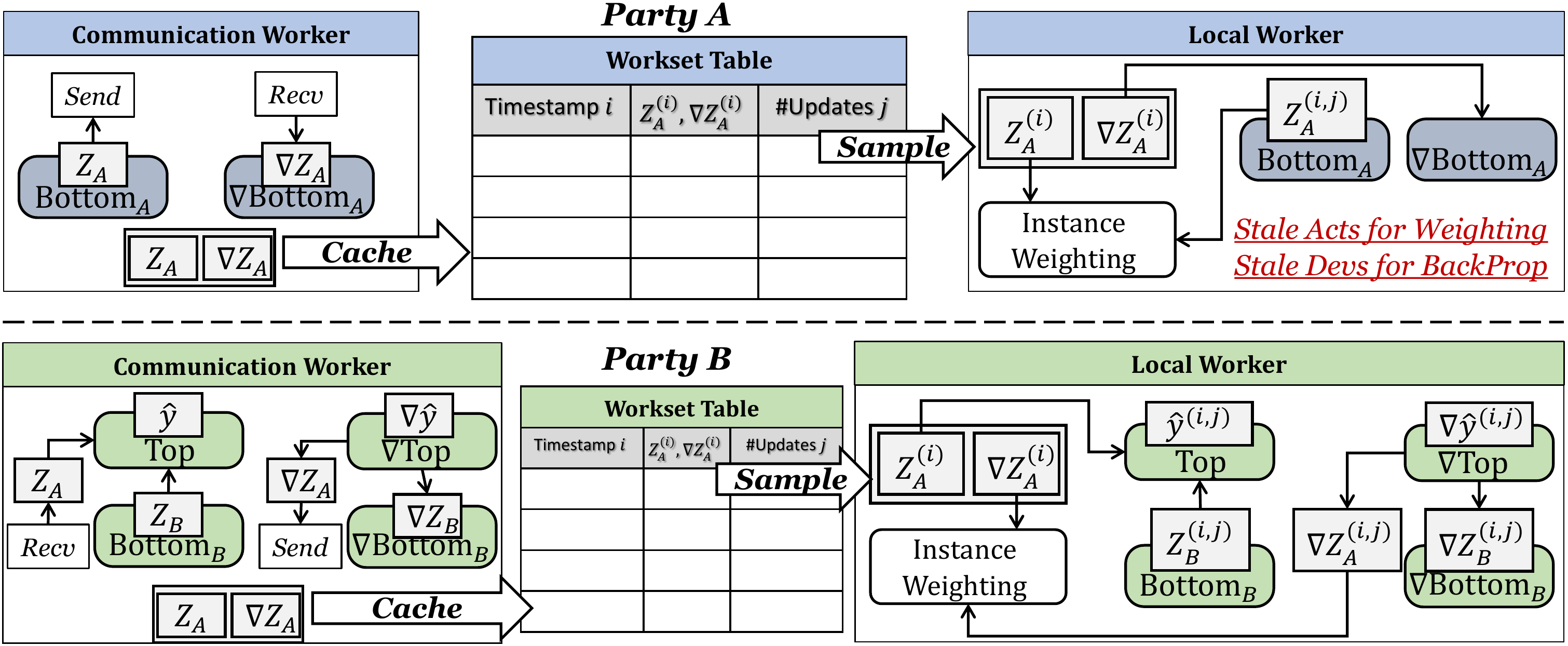}
\caption{Overview of \system.}
\label{fig:overview}
\end{figure*}

\section{\system}
\label{sec:method}

In this section, we first introduce \system, 
a novel VFL framework that enables local updates via cached statistics. 
Then, we propose two techniques to 
address the two drawbacks of local updates in VFL accordingly. 
We first present the frequently used notations: 
\begin{itemize}[leftmargin=*]
\item $B$: size of mini-batches;
\item $R$: maximum number of updates for each mini-batch; 
\item $W$: size of the workset table (number of cached mini-batches); 
\item $\xi$: the threshold in our instance weighting mechanism; 
\item {\small$\za^{(i)}, \nabla\za^{(i)}$}: the forward activations and backward derivatives of \host of the $i$-th batch; 
\item {\small$\za^{(i,j)}, \nabla\za^{(i,j)}$}: the forward activations and backward derivatives of \host of the $i$-th batch at the $j$-th updates. 
\end{itemize}

\subsection{Overview}
Figure~\ref{fig:overview} illustrates the overview of \system. 
There are three major components, as introduced below. 

\textit{The workset table} 
is in charge of the maintenance of 
the cached, stale statistics. 
It records two ``clocks'' for each cached batch: 
(i) the timestamp when this batch is inserted, 
and (ii) the current number of updates done by this batch. 
The capacity of the table is $W$, 
i.e., during the insertion at time $i$, 
we discard cached batches inserted before time $i - W + 1$ 
to control the maximum staleness. 
Whenever a cached batch is sampled for one local update, 
we increment the second clock correspondingly. 
Cached batches that reach the maximum number of updates, 
i.e., $R$, 
will be dropped as well. 

\textit{The communication worker} 
is responsible for 
the message transmission between parties. 
In the beginning, both parties agree on 
the ordering of mini-batches. 
For each batch, the communication workers 
execute the conventional forward and backward propagation 
to exchange {\small$\za$} and {\small$\nabla\za$}. 
Then, both {\small$\za$} and {\small$\nabla\za$} 
are inserted into the workset table 
in order to execute local updates. 
The timestamp $i$ is set as 
the current number of communication rounds, 
i.e., the current number of processed batches. 

\textit{The local worker} 
performs local updates via the cached statistics. 
For each local step, the local worker of each party 
chooses one cached entry 
{\small$\langle i, \za^{(i)}, \nabla\za^{(i)}, j \rangle$} 
from the workset table, 
where $i$ identifies the corresponding mini-batch 
and $j$ denotes how many times it has been used. 
Unlike FedBCD that trains on the same mini-batch 
for multiple consecutive steps, 
\ul{we exploit a round-robin local sampling strategy 
to reduce the stochastic variance}, 
which will be described in Section~\ref{sec:local_sampling}. 
In addition to executing the local propagation, 
\ul{we develop a mechanism that weights the instances 
to mitigate the impact brought by staleness}. 
For instance, as depicted in Figure~\ref{fig:overview}, 
\host feeds the stale derivatives {\small$\nabla\za^{(i)}$} to 
the backward propagation process
and utilizes the stale activations {\small$\za^{(i)}$} to 
measure the weights of instances. 
We will introduce the weighting mechanism in depth 
in Section~\ref{sec:ins_weighting}. 
Furthermore, we let the two types of workers run concurrently 
to make full use of 
both computation and communication resources.

\subsection{Round-Robin Local Sampling}
\label{sec:local_sampling}

It is well known that stochastic optimization algorithms 
have inherent variance in the mini-batch gradients
~\cite{pmlr-v70-allen-zhu17a,svrg,lazygcn}. 
As shown in Figure~\ref{fig:mini_batch}, 
if each mini-batch is used for consecutive steps, 
the stochastic variance will be accumulated 
and hinder the convergence. 
Such stochastic variance can be mitigated 
if we apply the mini-batches uniformly, 
e.g., by alternating the mini-batches. 
Motivated by this, we introduce 
a workset table that stores $W > 1$ mini-batches 
to allow uniform local sampling. 
It is worthy to note that since FedBCD repetitively 
uses the same mini-batch for consecutive local updates, 
it is analogous to the simplest case of $W = 1$. 

Nevertheless, we observe that 
simply letting $W > 1$ still cannot deal with 
the repetitive training issue. 
The reason is that the training workers are running 
too much faster than the communication workers. 
For instance, the top row of 
Figure~\ref{fig:local_sampling} 
depicts a case where $W=R=3$. 
Whenever a mini-batch is inserted into the workset, 
the training worker will immediately 
finish the three local updates 
before the next mini-batch is ready. 
As a result, the local sampling downgrades to 
the repetitive training pattern. 

To solve this problem, we devise a simple but effective 
round-robin local sampling strategy, 
which enforces that each mini-batch in the workset 
cannot be sampled again in the next $W - 1$ steps. 
As depicted in the bottom row of 
Figure~\ref{fig:local_sampling}, 
with the round-robin strategy, 
the training worker samples the mini-batches 
one by one, guaranteeing the uniformity. 
Although bubbles are incurred 
in the first $W - 1$ communication rounds, 
they are negligible since 
$W$ is far smaller than 
the total number of communication rounds.

\begin{figure}[!t]
\begin{minipage}{.45\textwidth}
\centering
\includegraphics[width=3in]{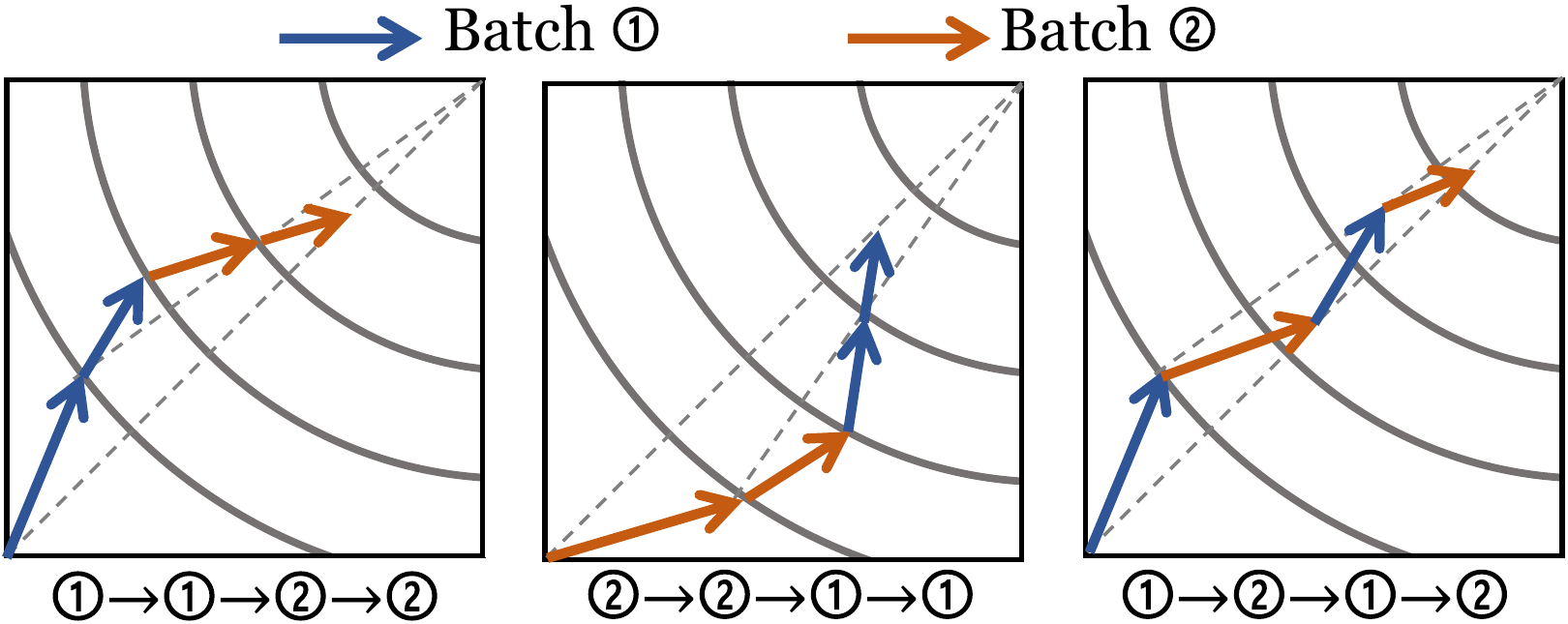}
\captionof{figure}{An illustration of the impact of mini-batches ordering. 
Both of the two mini-batches 
(indicated in blue and orange, respectively) 
are trained for two steps. 
Training in an alternating way (right-most) 
converges better than 
training in a repetitive way (left-most and middle).}
\label{fig:mini_batch}
\end{minipage}
\begin{minipage}{.45\textwidth}
$ $
\end{minipage}
\begin{minipage}{.45\textwidth}
\centering
\includegraphics[width=3in]{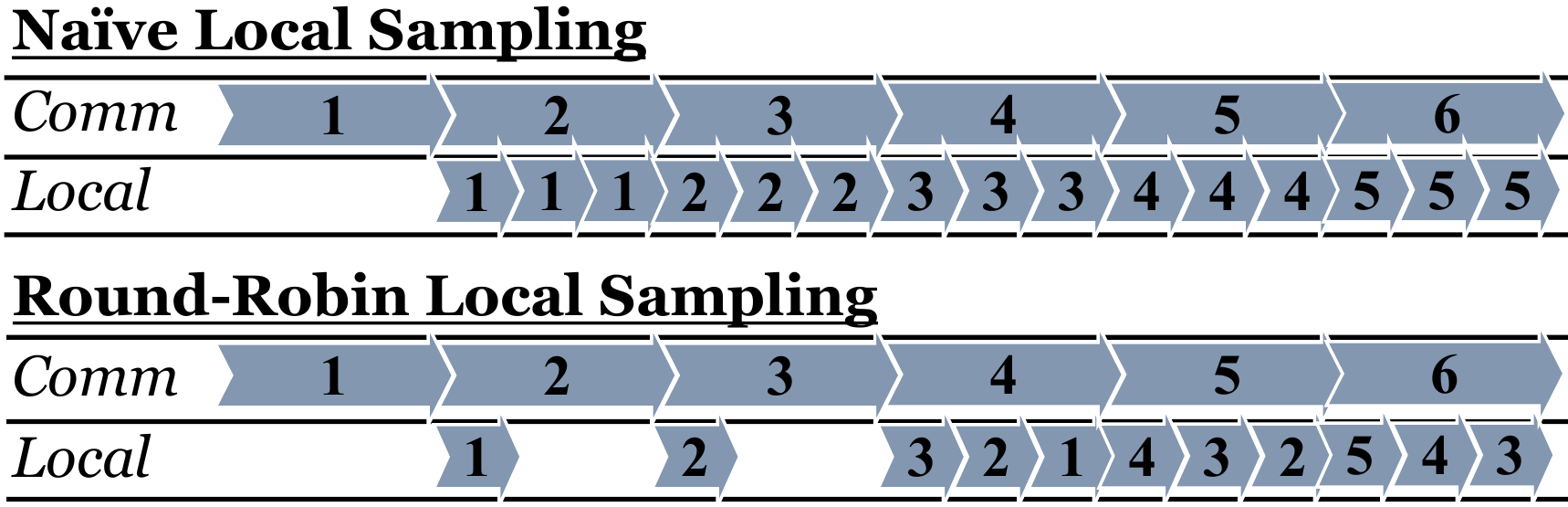}
\captionof{figure}{Na\"ive vs. round-robin local sampling.}
\label{fig:local_sampling}
\end{minipage}
\end{figure}

\mysubsubsection{Discussion}
Readers might suspect that whether the convergence 
would be slowed down as the maximum staleness 
of \system becomes $W \times R$, 
whilst that of FedBCD is only $R$. 
However, as we will analyze in Section~\ref{sec:converge_analysis} 
and evaluate in Section~\ref{sec:expr}, 
although the round-robin strategy increases the maximum staleness, 
the fairness in local sampling is beneficial to convergence 
from two aspects. 
First, given the same $R$, training with $W > 1$ 
converges faster than $W = 1$. 
Second, \system can support a larger $R$ 
whilst FedBCD diverges, 
so the resource utilization can be improved as well. 

In addition, there are many alternatives 
to achieve uniform sampling, 
such as randomly drawing the instances inside the workset table. 
However, we treat each pre-determined mini-batch as a whole 
to be more implementation-friendly --- 
since the mini-batches are inserted at different steps, 
it is easier to maintain the ``clocks'' of each instance. 
Furthermore, we can randomly shuffle 
the entire training dataset before training 
to ensure that the instances inside the workset table 
are also in random order.

\begin{algorithm}[!t]
\caption{\small{Procedures of 
local update with our instance weighting mechanism. 
$\cos\xi$ is a pre-defined threshold 
and $\odot$ denotes element-wise multiplication.}}
\small
\label{alg:local_update}
%\DontPrintSemicolon
\SetKwFunction{InsWeight}{InsWeight}
\SetKwFunction{FP}{ForwardProp}
\SetKwFunction{BP}{BackwardProp}
\SetKwFunction{Loss}{Loss}
\SetKwFunction{LocalUpdateA}{LocalUpdatePartyA}
\SetKwFunction{LocalUpdateB}{LocalUpdatePartyB}
\SetKwProg{Fn}{Function}{:}{}

\Fn{\InsWeight{$\textup{ad hoc } V^{(i,j)}, 
				\textup{stale } V^{(i)}$}}
{
	$weights \gets \cos(V^{(i,j)}, V^{(i)}, axis=1)$\;
	$weights[weights < \cos\xi] \gets 0$\;
	\Return $weights$
}

\Fn{\LocalUpdateA{$i, \za^{(i)}, \nabla\za^{(i)}, j$}}
{
	$\za^{(i,j)} \gets \FP(\xa^{(i)})$\;
	$ins\_weights \gets \InsWeight(\za^{(i,j)}, \za^{(i)})$\;
	$\BP(ins\_weights \odot \nabla\za^{(i)})$\;
}

\Fn{\LocalUpdateB{$i, \za^{(i)}, \nabla\za^{(i)}, j$}}
{
	$\hat{y}^{(i,j)} \gets \FP(\za^{(i)}, \xb^{(i)})$\;
	$loss^{(i,j)} \gets \Loss(y^{(i)}, \hat{y}^{(i,j)}, axis=1)$\;
	$\nabla\za^{(i,j)} \gets 
		\partial loss^{(i,j)} / \partial \za^{(i)}$\;
	$ins\_weights \gets \InsWeight(\nabla\za^{(i,j)}, \nabla\za^{(i)})$\;
	$\BP(ins\_weights \odot loss^{(i,j)})$\;
}
\end{algorithm}

\subsection{Instance Weighting Mechanism}
\label{sec:ins_weighting}

Undoubtedly, approximation with the stale statistics causes 
errors in the estimated model gradients. 
Although there have been various works 
discussing that gradient optimization algorithms 
are robust to certain noises in gradients, 
the model performance will be harmed inevitably 
when the staleness is large. 
To tackle this problem, we desiderate 
a fine-grained mechanism that measures the reliability of 
each cached instance according to its staleness. 

As we will analyze in Section~\ref{sec:converge_analysis}, 
the convergence speed 
is highly related to the angles between 
the approximated gradients calculated using stale statistics 
and the expected gradients calculated using ad hoc statistics. 
Therefore, it is a straightforward idea to 
leverage the cosine similarity to 
measure the staleness. 
Nevertheless, it is impossible to 
calculate the expected gradients 
for local updates since the ad hoc statistics 
are not transmitted between parties, 
making the cosine similarity unmeasurable. 
Consequently, we heuristically 
exploit the cosine similarity measurement 
to the forward activations and backward derivatives, 
rather than the model gradients. 

To be specific, our instance weighting mechanism is 
shown in Algorithm~\ref{alg:local_update}. 
We measure the cosine similarities 
between the ad hoc statistics and the stale ones, 
and use the similarities as instance weights. 
For \host, since the ad hoc forward activations 
are computed for each local update, 
we measure the cosine similarities 
by {\small$\cos(\za^{(i,j)}, \za^{(i)})$}, 
where the cosine operation applies on every row 
for each instance individually. 
In symmetric, \guest measures the cosine similarities 
via the backward derivatives\footnote{
Although the algorithm incurs the extra computation for 
{\small$\nabla\za^{(i,j)}$}, which is not needed 
in non-weighted local updates, 
the extra cost is negligible and worthwhile.}, 
i.e., {$\cos(\nabla\za^{(i,j)}, \nabla\za^{(i)})$}. 
Furthermore, we introduce a lower bound $\cos\xi$ 
as the threshold, 
i.e., for instances whose similarities 
are smaller than $\cos\xi$, 
their weights will be set as zero. 
Finally, the backward propagation is executed 
in a weighted manner so that 
the model gradients will be computed 
in the weighted-averaged fashion.

\mysubsubsection{Discussion}
To help readers better understand 
the rationality of such a heuristic method, 
we use a fully connected layer 
{\small$z_{out} = z_{in} \theta$} as an example, 
where {\small$\theta \in \mathbb{R}^{d_{in} \times d_{out}}$} 
denotes the model weights 
and {\small$z_{in} \in \mathbb{R}^{d_{in}}, 
z_{out} \in \mathbb{R}^{d_{out}}$} 
denote the inputs and outputs of one instance, respectively. 
During the backward propagation, 
the model gradients are computed via 
the outer product between forward activations and 
backward derivatives, i.e., 
{\small$\nabla \theta = z_{in}^T \nabla z_{out}$}. 
For \host, suppose the stale derivatives 
are denoted as {\small$\widetilde{\nabla} {z}_{out}$}, 
it is easy to know that 
{\small$\cos(
    \nabla \theta, 
    \widetilde{\nabla} \theta
) = \cos(
	\nabla z_{out}, 
	\widetilde{\nabla} {z}_{out}
)$}\footnote{Here we assume 2-dimensional matrices 
({\small$\nabla \theta, \widetilde{\nabla} \theta$}) 
are flattened before the computation of cosine similarities.}, 
where {\small$\widetilde{\nabla} \theta = z_{in}^T \widetilde{\nabla} {z}_{out}$}. 
Similarly, we have the same property for \guest. 
Although this is only valid for fully connected layers, 
it provides a hint that measuring the similarities 
via forward activations and backward derivatives 
is reasonable and achieves promising results.

\section{Analysis}
\label{sec:analysis}

%In this section, we theoretically analyze 
%the convergence properties 
%and security guarantees of \system. 

\subsection{Convergence Analysis}
\label{sec:converge_analysis}
When applying local updates in VFL training, 
we are actually solving the problem defined in Equation~\ref{eq:vfl} 
with approximated gradients. 
For simplicity, we re-write it as a more general 
(non-convex) empirical risk minimization (ERM) problem, i.e., 
\begin{equation*}
\small
\centerline{% disable math mode
    $\min_{{\theta\in\mathbb{R}^d}} f(\theta) = 
	    \frac{1}{N} \sum_{k \in \mathcal{D}} f_k(\theta), 
    \;
    {\theta}_{t+1} = {\theta}_{t} - \eta {\tilde{g}}_{t}$,
}
\end{equation*}
where $\eta$ is the step size and 
step $t$ is denoted as the number of updates. 
It is worthy to note that since the models of different parties 
may not be updated for the same number of steps 
due to the local update technique, 
the learning task in \system can be reckoned as 
solving two ERM problems separately. 
To analyze the problem, we denote three types of gradients: 
\begin{itemize}[leftmargin=*]
\item 
{\small$\nabla f(\theta_t) = \frac{1}{N} \sum_{k\in\mathcal{D}} \nabla f_k(\theta_t)$} 
is the full-batch gradient. 
\item 
{\small$g_t = \frac{1}{B} \sum_{k\in\mathcal{B}} \nabla f_k(\theta_t)$} 
is the mini-batch gradient of the batch {\small$\mathcal{B}$} 
subsampled from the workset table {\small$\mathcal{W}$}, 
i.e., {\small$\mathcal{B} \subseteq \mathcal{W} \subseteq \mathcal{D}$}. 
\item 
{\small$\tilde{g}_t$} is the approximated version of {\small$g_t$} 
using stale statistics. 
\end{itemize}

Compared to vanilla SGD training, 
the analysis in our work is more challenging 
due to the two drawbacks of local updates in VFL: 
(1. variance problem) 
$g_t$ is sampled over the workset table instead of the entire training data; 
(2. approximation errors) 
$\tilde{g}_t$ in local updates is approximated 
using the stale information, 
which cannot guarantee unbiasedness, 
i.e., $\mathbb{E}[\tilde{g}_t] \neq \nabla f(\theta_t)$.
Consequently, the expected errors 
can be dissected into two terms: 
\begin{equation*}
\small
	\mathbb{E}\left[\lVert\tilde{g}_t - \nabla f(\theta_t)\rVert^2\right]
	\leq 2 \mathbb{E}\left[\lVert g_t - \nabla f(\theta_t)\rVert^2\right] 
	+    2 \mathbb{E}\left[\lVert\tilde{g}_t - g_t\rVert^2\right],
\end{equation*}
where the first term depicts the stochastic variance due to 
mini-batch sampling 
whilst the second term comes from the approximation 
via stale information. 
Then, we make two assumptions. 
\begin{assumption}{1}
\label{asp:gradient}
For each step $t$, we make these two assumptions: 
\\(1. Uniform sampling) 
$\mathcal{W}$ is uniformly subsampled from $\mathcal{D}$ 
and $\mathcal{B}$ is uniformly subsampled from $\mathcal{W}$.
\\(2. Correct direction) 
The angles between $\tilde{g}_t, g_t$ are bounded, i.e., 
there exists a constant $\rho \in (0,1]$ s.t. 
$\cos(\tilde{g}_t, g_t) \geq \rho$.
\end{assumption}
\begin{assumption}{2}
\label{asp:empirical}
For $\forall k \in \{1, 2, ..., N\}, \theta, \theta^\prime \in \mathbb{R}^d$, 
we make these assumptions following~\citet{pmlr-v70-allen-zhu17a}:
\\(1. $L$-Lipschitz)
$\lVert{\nabla} f_k({\theta}) - {\nabla} f_k({\theta^\prime})\rVert 
\leq L \lVert{\theta}-{\theta^\prime}\rVert$;
\\(2. Bounded moment)
$\mathbb{E}[\lVert{\nabla} f_k({\theta})\rVert^2] \leq \sigma^2$;
\\(3. Existence of global minimum)
$\exists \theta^*$ s.t. $f({\theta}) \geq f(\theta^*)$.
\end{assumption}
Assumption~\ref{asp:gradient} ensures the unbiasedness
of $\mathbb{E}[g_t] = \nabla f(\theta_t)$ and 
guarantees that the descending directions of the approximated gradients 
are convincing. 
Assumption~\ref{asp:empirical} consists of the standard assumptions 
widely used in the convergence analysis of stochastic optimization. 
We present our theoretical result in Theorem~\ref{thm:convergence}. 
\begin{theorem}
\label{thm:convergence}
Under Assumption~\ref{asp:gradient},\ref{asp:empirical} 
and setting the step size appropriately, 
after $T$ steps, select $\bar{\theta}_{T}$ randomly from 
$\{\theta_t\}_{t=0}^{T-1}$. 
With probability at least $1 - \delta$, we have 
{\small$\mathbb{E}\left[\lVert{\nabla f(\bar{\theta}_T)}\rVert^2\right] 
\leq O(\sqrt{\Delta / T})$}, where 
\begin{equation*}
	\Delta = \frac{L^2 \log\left({2d}/{\delta}\right)}{B} 
	\left(1 + \frac{1}{W}\right) 
	+ \sigma^2 (2 - \rho).
\end{equation*}
\end{theorem}

\mysubsubsection{Remark}
Note that the convergence analysis by~\citet{fedbcd} relies on 
a strong assumption on the approximated gradients\footnote{\citet{fedbcd} 
assume the approximated gradients are unbiased 
to analyze the convergence of FedBCD, 
whilst our analysis does not need such a strong assumption.}, 
which fails to reveal the negative impacts brought by local updates. 
In contrast, our analysis relies on more relaxed assumptions and 
uncovers the variance problem and approximation errors in local updates. 
This also sheds light on 
how the local sampling and 
instance weighting influence the convergence. 
On the one hand, with a larger $W$, 
the stochastic variance becomes smaller (i.e., smaller $\Delta$) 
and thus benefits the convergence. 
On the other hand, if the stale information is more reliable, 
which gives a higher $\rho$, 
the convergence would be faster as well. 
Nevertheless, it is worthy to notice that 
there exists a tradeoff since 
$\rho$ decreases w.r.t. $W$ and $R$ 
due to the larger staleness. 

\mysubsubsection{Communication Complexity}
Compared with the sub-linear rate $O(1/\sqrt{T})$ of 
vanilla SGD training, 
our framework introduces an extra factor $O(\sqrt{\Delta})$. 
In other words, vanilla SGD training needs $O(T)$ steps 
whilst we need $O(\Delta \times T)$ steps 
to achieve similar convergence. 
However, the communication complexity of our framework 
can be reduced to $O(\Delta/R)$ times of 
that of vanilla SGD training. 
In Section~\ref{sec:expr}, we will empirically show that 
\system requires much fewer communication rounds 
to achieve the same model quality.

\subsection{Security Analysis}
\label{sec:security_analysis}

Obviously, our framework shares the same security guarantees as 
vanilla VFL training as introduced in Section~\ref{sec:bg}, 
since we do not need to exchange values other than 
activations and derivatives. 
We refer interested readers to the previous studies 
for the detailed analysis~\cite{async_vfl_1,async_vfl_2,fedbcd,vafl}.
Moreover, because our framework can effectively 
reduce the number of communication rounds 
as shown in our experiments, 
the number of messages transmitted between parties 
would become smaller, which conveys 
stronger privacy preservation.

\section{Implementation and Evaluation}
\label{sec:expr}

\subsection{Setup}
\mysubsubsection{Implementation}
We implement \system on top of TensorFlow~\cite{tensorflow} 
and gRPC. 
We first develop lightweight message queues to 
support the \texttt{send} and \texttt{recv} primitives 
for cross-party communication. 
Then, an \texttt{FLGraph} abstraction is proposed to 
automatically generate the backward computation graph and the local update operation 
after the users have defined the bottom and top models. 
\system has been deployed to 
the collaborative recommendation tasks 
of Tencent Inc. 
To be specific, 
Tencent works as \host and provides user features, 
and \guest is an e-commercial company that 
owns item features and labels 
(such as clicks or conversions), 
but lacks a rich set of user features. 
In the preprocessing phase, the private data of both companies 
are aligned as user-item pairs. 
Then, \system trains VFL models 
to fit the labels by the features of both parties. 

\mysubsubsection{Experimental Setup}
All experiments are carried out on two servers 
where each server acts as one party. 
For each party, we run the training tasks with 
one NVIDIA V100 GPU. 
The WAN bandwidth between the two servers is 300Mbps. 

\textbf{Models.} 
We focus on the deep learning recommendation models (DLRMs) 
since it is one of the most common neural network architectures 
in real-world VFL collaborations. 
To be specific, we consider two widely-used DLRMs, 
which are Wide and Deep (WDL) and 
Deep Structured Semantic Model (DSSM).

\begin{table}[!t]
\small
\centering
\caption{\small{Description of datasets for end-to-end evaluation.}}
\begin{tabular}{|c|c|c|c|}
\hline
\textbf{Dataset} & \textbf{\#Instances (train/test)}
& \textbf{\#Fields (A/B)} \\
\hline
\hline
\textit{Criteo} & 41M/4.5M & 26/13 \\
\hline
\textit{Avazu} & 36M/4.0M & 14/8 \\
\hline
\textit{D3} & 32M/3.5M & 25/18 \\
\hline
\end{tabular}
\label{tb:dataset}
\end{table}

\textbf{Datasets.} 
Three datasets are used in our experiments, 
as listed in Table~\ref{tb:dataset}. 
\textit{Criteo}\footnote{\url{http://labs.criteo.com/2014/02/kaggle-display-advertising-challenge-dataset/}} 
and \textit{Avazu}\footnote{\url{https://www.kaggle.com/c/avazu-ctr-prediction/data}} 
are two large-scale public datasets. 
Whilst \textit{D3} is a productive dataset 
in Tencent Inc. 
As our major focus is the training phase, 
we assume all datasets have been aligned. 

\textbf{Protocols.} 
We train the models using the AdaGrad optimizer~\cite{adagrad} 
with a batch size of 4096. 
The output dimensionality of $\za$ 
is set as 256. 
To achieve a fair comparison, 
we tune the learning rate for 
vanilla VFL training and apply the same  
to our work directly. 
For each experiment, we run three trials 
and report the mean and standard deviation (stddev).

\begin{figure*}
\centering
\subfigure[Impact of local update.]{
\label{fig:sensitivity_update}
\includegraphics[width=.24\textwidth]{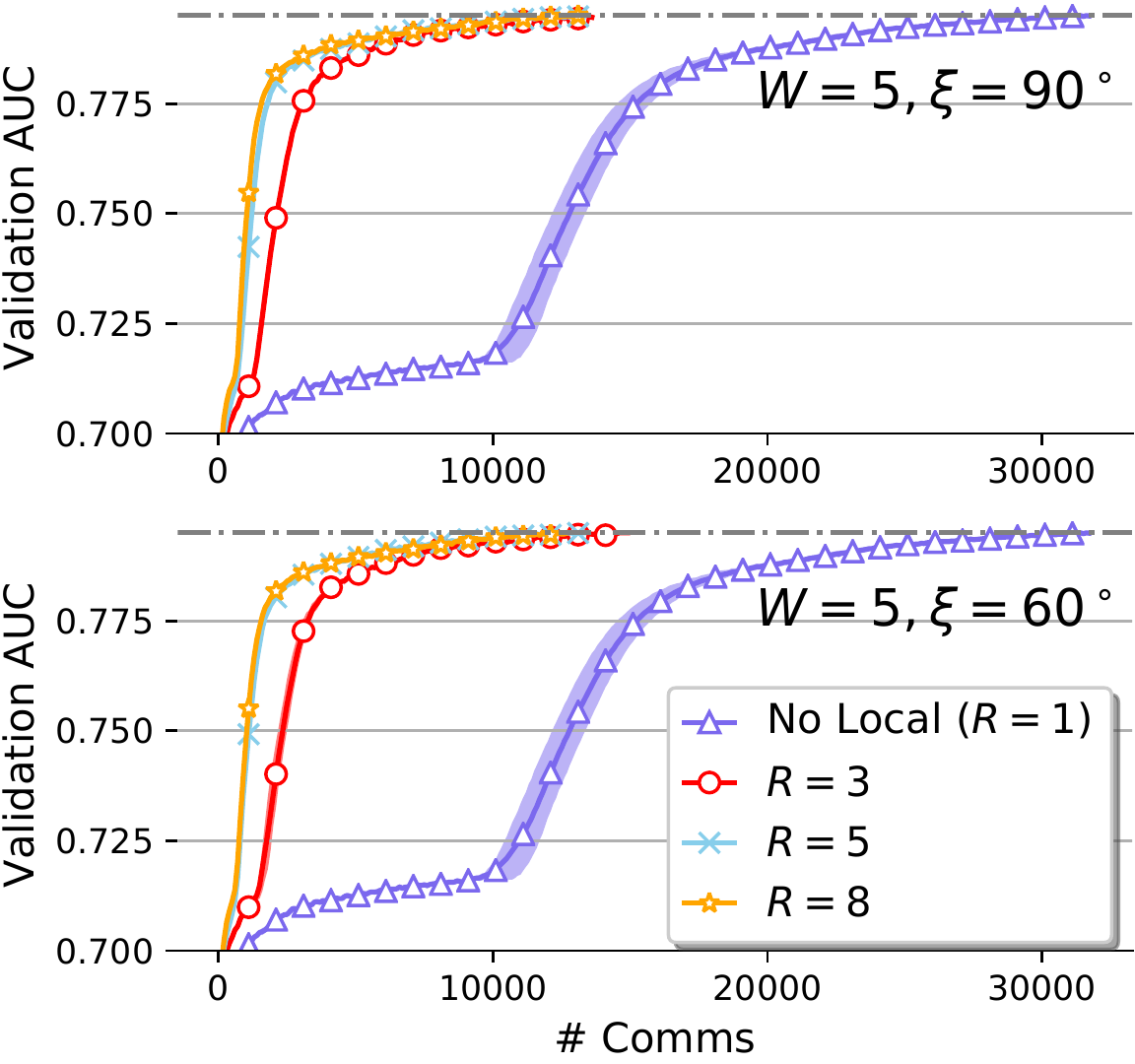}}
\subfigure[Impact of local sampling.]{
\label{fig:sensitivity_workset}
\includegraphics[width=.24\textwidth]{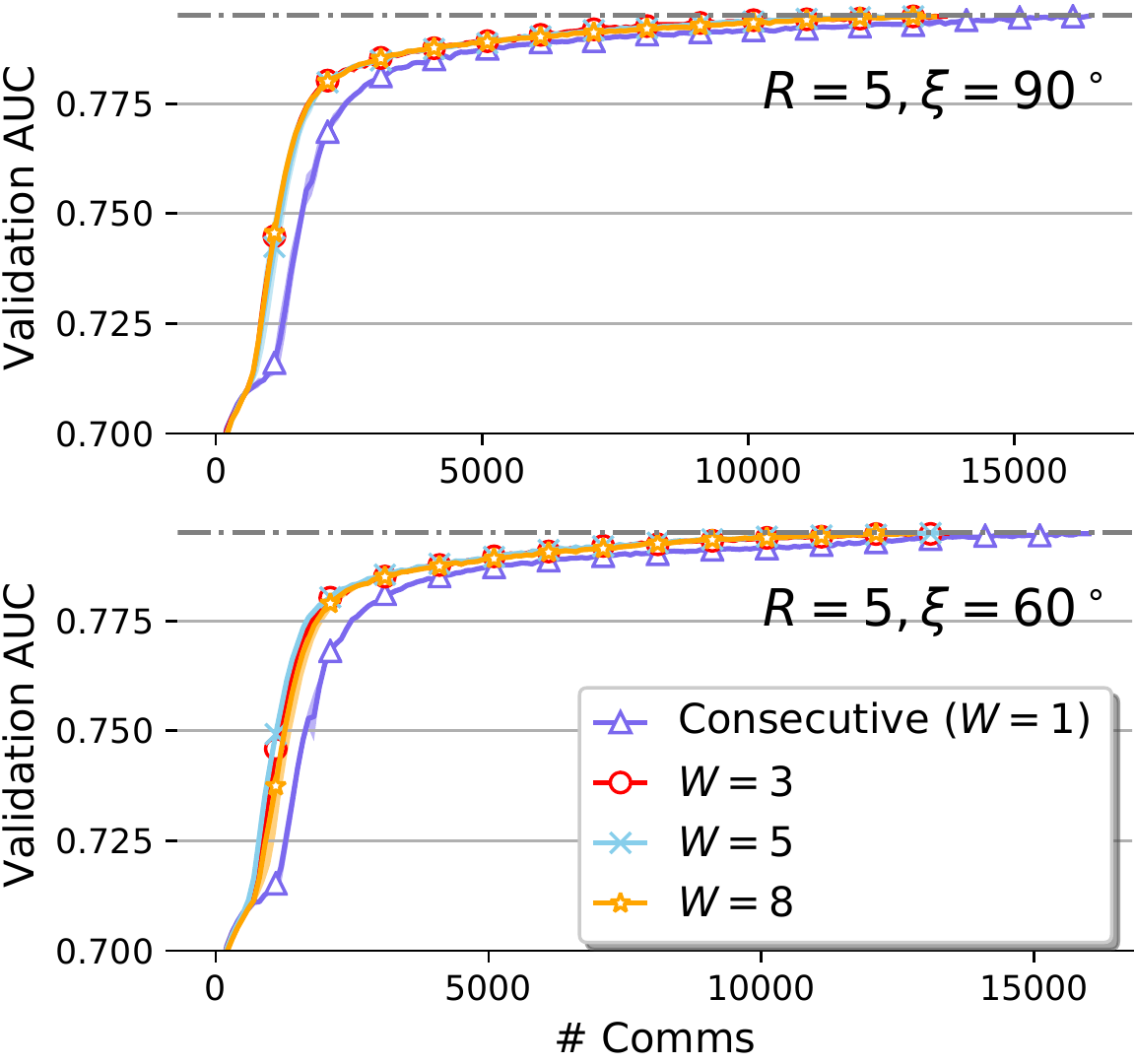}}
\subfigure[Impact of instance weighting.]{
\label{fig:sensitivity_weight}
\includegraphics[width=.24\textwidth]{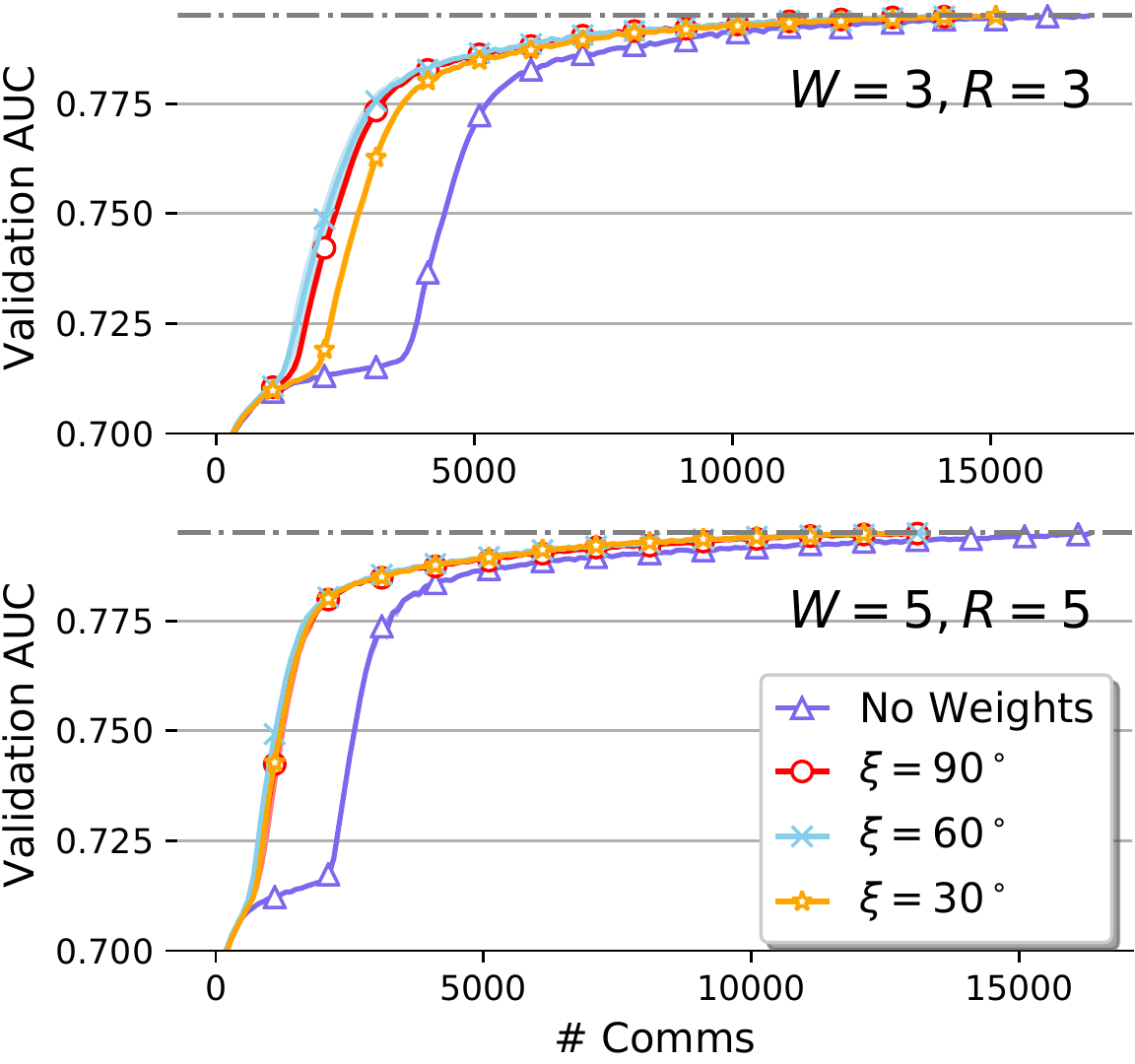}}
\subfigure[Quantiles of cosine similarities.]{
\label{fig:cosine_similarity}
\includegraphics[width=.24\textwidth]{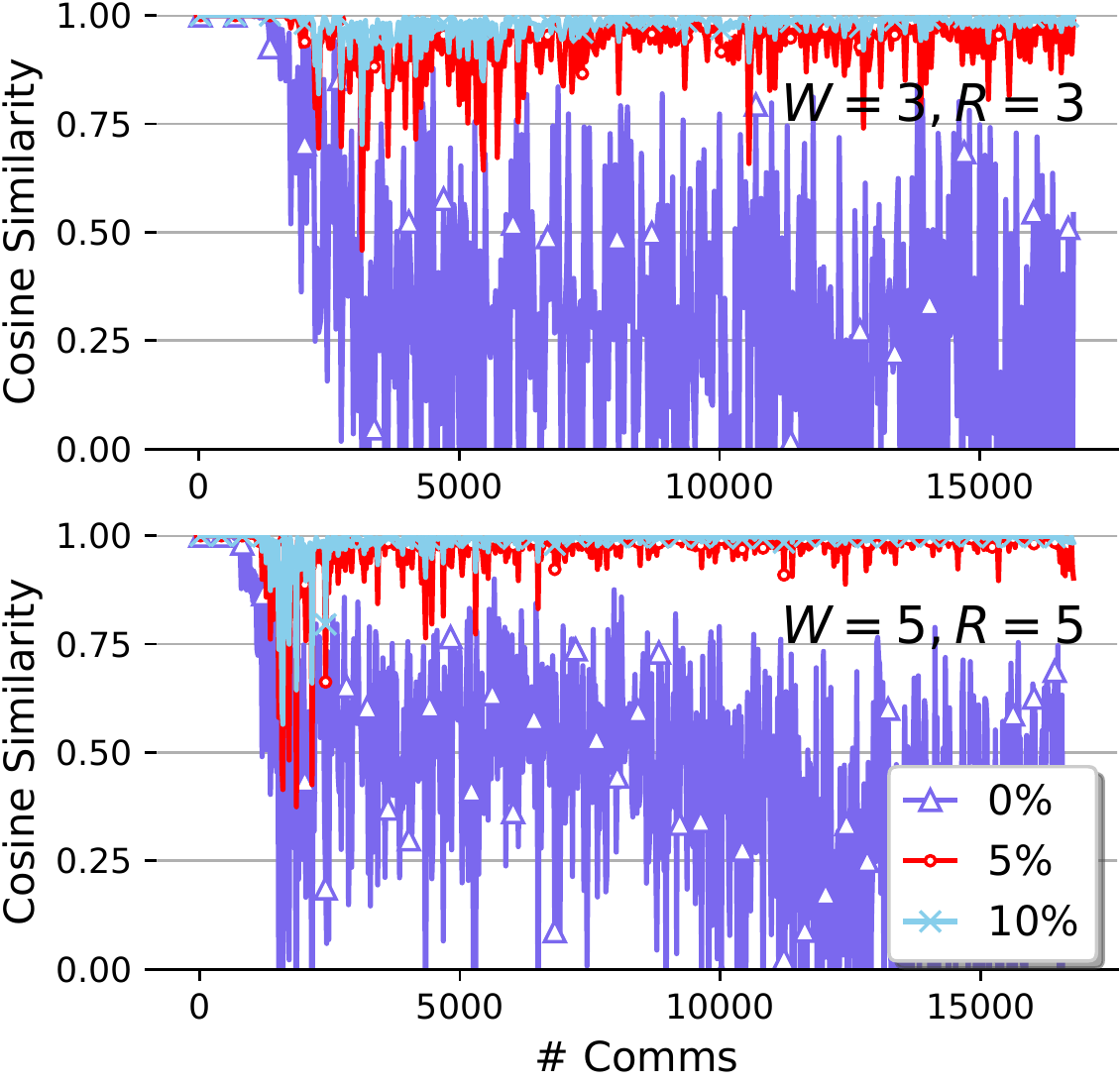}}
\caption{\small{Validation AUC metrics in terms of 
the number of communication rounds. 
For each technique, we conduct experiments with and without 
the technique and vary the corresponding hyper-parameter 
under two different settings.}}
\label{fig:sensitivity}
\end{figure*}

\begin{table*}[!t]
\centering
\small
\caption{\small{Number of communication rounds
(mean and stddev) 
required to reach the same model performance 
in Figure~\ref{fig:sensitivity}.}}
\begin{tabular}{|c|c|c|c|c|}
\hline
\textbf{Local Update}
& {No Local $(R=1)$} 
& {$R=3$} 
& {$R=5$} 
& {$R=8$} 
\\
\hline
$W=5,\xi=90^\circ$ 
& $31540 \pm 631.19$ 
& $14000 \pm 424.26 (\downarrow 55.61\%)$
& $12767 \pm 654.90 (\downarrow 59.52\%)$
& $12733 \pm 410.96 (\downarrow 59.63\%)$ 
\\
\hline
$W=5,\xi=60^\circ$ 
& $31540 \pm 631.19$ 
& $14333 \pm 329.98 (\downarrow 54.56\%)$
& $12767 \pm 286.74 (\downarrow 59.52\%)$
& $12500 \pm 424.26 (\downarrow 60.37\%)$ 
\\
\hline
\hline
\textbf{Local Sampling}
& {Consecutive $(W=1)$} 
& {$W=3$} 
& {$W=5$} 
& {$W=8$} 
\\
\hline
$R=5,\xi=90^\circ$ 
& $16400 \pm 244.95$ 
& $13333 \pm 339.94 (\downarrow 18.70\%)$ 
& $12767 \pm 654.90 (\downarrow 22.15\%)$ 
& $13333 \pm 249.44 (\downarrow 18.70\%)$ 
\\
\hline
$R=5,\xi=60^\circ$ 
& $15967 \pm 339.94$ 
& $12700 \pm 141.42 (\downarrow 20.46\%)$ 
& $12767 \pm 286.74 (\downarrow 20.04\%)$ 
& $12667 \pm 419.00 (\downarrow 20.67\%)$ 
\\
\hline
\hline
\textbf{Instance Weighting}
& {No Weights} 
& {$\xi=90^\circ$} 
& {$\xi=60^\circ$} 
& {$\xi=30^\circ$} 
\\
\hline
$W=3,R=3$ 
& $16767 \pm 492.16$ 
& $14233 \pm 235.70 (\downarrow 15.11\%)$
& $15200 \pm 163.30 (\downarrow 9.36\%)$
& $14667 \pm 94.28 (\downarrow 12.52\%)$ 
\\
\hline
$W=5,R=5$ 
& $16467 \pm 492.16$ 
& $12767 \pm 286.74 (\downarrow 22.47\%)$
& $12567 \pm 205.48 (\downarrow 23.68\%)$
& $12767 \pm 654.90 (\downarrow 22.47\%)$ 
\\
\hline
\end{tabular}
\label{tb:sensitivity}
\end{table*}

\subsection{Ablation Study and Sensitivity}
\label{sec:expr_sensitivity}
Compared with the vanilla VFL training, 
\system contains three additional techniques, 
i.e., the local update technique, 
the uniform local sampling strategy, 
and the instance weighting mechanism 
(parameterized by $R, W, \xi$, respectively). 
We first investigate the impact of each technique 
and assess the sensitivity w.r.t. the hyper-parameters. 
To be specific, we train WDL on the $Criteo$ dataset 
and measure the number of communication rounds taken 
to the same target AUC metric. 
The convergence curves are shown 
in Figure~\ref{fig:sensitivity} 
and the required communication rounds 
are listed in Table~\ref{tb:sensitivity}.

\mysubsubsection{Impact of Local Update} 
We first assess the effect of the local update technique 
by varying $R$. 
The results are shown 
in Figure~\ref{fig:sensitivity_update}. 
In general, compared with the vanilla VFL training, 
applying the local update technique can significantly 
reduce the number of communication rounds 
required to reach the same model performance. 
For instance, when $R=3$, the number of 
communication rounds is reduced by around 55$\%$. 
When we increase $R$ to 5, 
the improvement increases to around 60$\%$ accordingly. 
Letting $R=8$ provides a faster convergence speed than 
that of $R=5$ in the early phase. 
However, due to the larger staleness, 
it takes a similar number of communication rounds 
to attain the same validation AUC. 

\mysubsubsection{Impact of Local Sampling}
To assess the impact of the local sampling strategy, 
we conduct experiments with various $W$. 
When $W=1$, each mini-batch will be utilized 
for multiple consecutive local steps, 
whilst we apply the round-robin local sampling strategy 
when $W>1$ to achieve uniform sampling. 
As shown in Figure~\ref{fig:sensitivity_workset}, 
the round-robin local sampling strategy 
consistently outperforms 
the consecutive counterpart --- 
the numbers of required communication rounds 
are reduced by 18-22$\%$. 
This matches our analysis that 
uniform sampling can help alleviate 
the stochastic variance in mini-batch gradients 
and therefore improve the convergence. 
Moreover, it is worthy to note that the convergence speed 
does not change significantly when $W \in \{3,5,8\}$, 
which verifies the robustness of \system 
against the choice of workset size.

\begin{figure*}
\centering
\includegraphics[width=\linewidth]{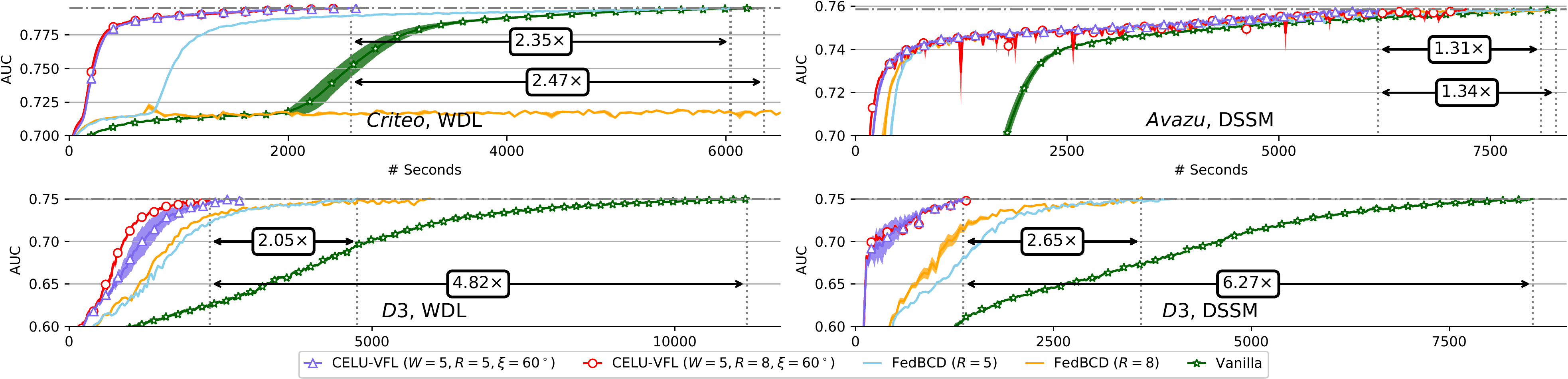}
\caption{\small{Validation AUC metrics in terms of 
running time (in seconds)}}
\label{fig:e2e}
\end{figure*}

\mysubsubsection{Impact of Instance Weighting}
Next, we explore the benefit brought by 
the instance weighting mechanism 
with different $\xi$. 
The results are presented 
in Figure~\ref{fig:sensitivity_weight}. 
In short, the instance weighting mechanism 
accelerates the convergence under both settings 
since it restricts the approximation errors 
incurred by the stale values. 
Under the first setting ($W=3,R=3$), 
we can save 9-15$\%$ of the communication rounds. 
Under the second setting ($W=5,R=5$), 
the improvement bumps to around $23\%$. 
This is a reasonable phenomenon --- 
with larger $W$ and/or $R$, 
the staleness of cached values increases and 
the approximation errors are enlarged as well, 
so the contribution of our weighting mechanism 
is more significant. 

We further plot the cosine similarities 
for better interpretability. 
To be specific, for each local update, 
we compute the quantiles of all similarities 
in the current batch 
(0\% quantiles correspond to the lowest similarities). 
The results in Figure{~\ref{fig:cosine_similarity}} show two facts: 
(1) Most of the stale statistics are reliable. 
For instance, over 90\% of the cosine similarities 
are greater than 0.5 even in the fast converging period, 
which supports our motivation that 
the model does not change significantly 
and the stale statistics can be re-used in local updates. 
(2) Our instance weighting mechanism is able to 
filter or penalize the outliers 
to accelerate the convergence.

\mysubsubsection{Summary}
The experiments demonstrate the power of 
each technique individually. 
Overall, all techniques have significant contributions 
to reducing the number of communication rounds 
compared with the counterparts without the techniques. 
Furthermore, our work is robust to different configurations 
(the choices of the hyper-parameters), 
which is vital to the applicability in practice.

\subsection{End-to-End Evaluation}
\label{sec:expr_e2e}
Then, we evaluate the end-to-end training performance 
of \system. 
To be specific, we compare \system with 
two competitors, which are FedBCD~\cite{fedbcd} 
and the vanilla VFL training 
(denoted as Vanilla). 
Following the results in Section~\ref{sec:expr_sensitivity}, 
we fix $W=5,\xi=60^\circ$ and let $R \in \{5,8\}$ 
for all experiments. 
We presents the convergence curves in Figure~\ref{fig:e2e} 
and discuss the results for different datasets below. 

\mysubsubsection{Results on Criteo}
We first train WDL on the \textit{Criteo} dataset. 
Thanks to the reduction in communication, 
\system achieves 2.47$\times$ of speedup compared with Vanilla. 
Whilst the performance of FedBCD is unsatisfactory. 
For $R=5$, although FedBCD runs faster than Vanilla 
in the short-term convergence, 
it slows down in the long-term convergence. 
Eventually, the gap between FedBCD and Vanilla is very small. 
The reason is that 
when the model approaches the (local) optimum point, 
the gradients become smaller in scale and therefore 
are more sensitive to 
the errors caused by local sampling and stale information. 
Worst still, FedBCD fails to converge for $R=8$ 
due to the large staleness. 
In contrast, \system has a better statistical efficiency 
owing to our techniques. 
As a result, \system can support a larger $R$ 
and outperform FedBCD by 2.35$\times$. 

\mysubsubsection{Results on Avazu}
We then evaluate the performance by training DSSM on 
the \textit{Avazu} dataset. 
First, in the early phase, 
both \system and FedBCD accelerate the convergence 
by a large extent compared with Vanilla 
with the help of local updates. 
However, similar to the results on \textit{Criteo}, 
FedBCD suffers from a slow convergence later, 
whilst \system converges properly 
and achieves 1.31-1.34$\times$ of improvement 
compared with the other competitors.  

\mysubsubsection{Results on Industrial Dataset}
Finally, we conduct experiments on 
the industrial dataset \textit{D3} 
to see how \system performs in real-world workloads. 
Overall, \system outperforms Vanilla by 4.82-6.27$\times$ 
on the two training tasks. 
FedBCD has a better performance on the industrial dataset, 
however, it is still 2.05-2.65$\times$ slower than \system. 
These empirical results demonstrate 
the superiority of our work 
in practical VFL applications.

\section{Related Works}
\label{sec:related}

\mysubsubsection{Federated learning}
In recent years, FL has aroused the interests 
from both academia and industry. 
Notably, VFL conveys a possibility to unite the data of different corporations 
and jointly build the models, 
which fits numerous real-world cross-enterprise collaborations
~\cite{secureboost,vf2boost,vfl_lr,vfl_label_protect,fdml}. 
Besides VFL, another important field of FL 
is the horizontal FL (HFL), 
where the datasets are horizontally partitioned
~\cite{fl_concepts,konevcny2016federated_opt,konevcny2016federated_learn,fedprox}. 
These works are orthogonal to ours since we focus on the VFL scenario.

\mysubsubsection{Asynchronous VFL}
In addition to the local update technique, 
various works address the cross-party communication cost 
in VFL via asynchronous learning. 
However, these works cannot be applied in our scenario. 
Below we categorize them into two lines. 

{(Asynchronous Workers)} 
The first kind of works tries to 
overlap computation and communication 
by using multiple training workers (threads) 
for each party~\cite{async_vfl_1,async_vfl_2}. 
The models are stored in a shared-memory 
and each worker individually reads and updates the models 
in a Hogwild! manner~\cite{hogwild}. 
By doing so, when some workers are waiting for 
the message transmission, 
the other workers can concurrently do the computation. 

Nevertheless, these methods mainly focus on linear models 
where the output dimensionality of the exchanged statistics is very low 
(e.g., only one for logistic regression and support vector machine). 
Due to the small message size, 
sending one message cannot make full use of the bandwidth, 
so their major bottleneck is actually 
the high network latency rather than the low network bandwidth. 
In contrast, the output dimensionality is much higher for many 
modern deep learning models, 
so the messages are much larger in size and 
the network bandwidth would be easily occupied 
when sending one message. 
Thus, letting multiple workers communicate concurrently 
would cause severe network congestion and even lead to failure. 
Consequently, these methods do not fit our scenario. 

{(Asynchronous Parties)} 
The second category assumes that there are two or more \host's 
and allows different parties to train 
asynchronously~\cite{vafl,fdml}. 
To be specific, \guest is treated as a server that holds the labels and 
stores the forward activations received from all \host's. 
During the training process, 
each \hosti individually samples a mini-batch 
and sends {\small$\zai$} to \guest. 
Upon receiving, \guest updates the local storage,  
computes and returns {\small$\nabla\zai$}. 
Finally, each \hosti individually performs a backward propagation 
to update {\small$\wbai$} without synchronizing with other parties. 
The spirit of this line of research lies with 
the asynchrony between \host's --- 
different \host's can train on different mini-batches 
to achieve asynchronous learning.

However, such methods are orthogonal to our work since 
we focus on the case where only one \host is involved, 
which is more commonly seen in real-world applications. 
Moreover, since these works still require to exchange 
forward activations and backward derivatives in every step, 
our work can be integrated with their methods easily 
by applying the local update technique 
for each \hosti individually. 
We would like to leave the extension to multi-party VFL training 
as our future work.

\section{Conclusion}
\label{sec:conc}

This paper focuses on the deficiency caused by heavy cross-party communication 
in VFL training. 
Specifically, we propose \system, an efficient VFL training framework 
with cached-enable local updates. 
To improve the statistical efficiency, we develop two novel techniques 
to achieve fairness in stale statistics sampling and 
measure the fine-grained reliability of the local updates, respectively. 
Finally, both theoretical analysis and empirical experiments verify 
the effectiveness of \system.

\begin{acks}
This work is supported by National Natural Science Foundation of China (NSFC) (No. 61832001), PKU-Tencent joint research Lab, and Beijing Academy of Artificial Intelligence (BAAI). Bin Cui is the corresponding author. 
\end{acks}

\balance

%%
%% The next two lines define the bibliography style to be used, and
%% the bibliography file.
\nocite{downpour_sandblaster,petuum,sketchml,skcompress,hogwild,easgd,het,het_gmp,snapshot_boosting,real_time_rec,cl4rec,gnn_rs,dp_survey,dp_dl,dp_norm_lap,robust_fl_edge,rookase,sadde}
\bibliographystyle{ACM-Reference-Format}
\bibliography{reference}

%%%%%%%%%%%%%%%%%%%%%%%%%%%%
\newif\ifappendix
\appendixtrue % Un-comment to include appendix
\ifappendix
\appendix
\section{Proofs}
\label{sec:proof}

\newcommand{\gp}{{g^\prime}}
\newcommand{\bg}{{\bar{g}}}
\newcommand{\tg}{{\tilde{g}}}

\newcommand{\norm}[1]{\left\lVert{#1}\right\rVert}
\newcommand{\normsqr}[1]{\norm{#1}^2}
\newcommand{\innerdot}[2]{\left\langle{#1},{#2}\right\rangle}

\newcommand{\E}[1]{\mathbb{E}\left[{#1}\right]}
\newcommand{\enorm}[1]{\E{\norm{#1}}}
\newcommand{\enormsqr}[1]{\E{\normsqr{#1}}}
\newcommand{\einnderdot}[2]{\E{\innerdot{#1}{#2}}}

This section provides proofs for the analysis in our paper. 

We first introduce an importance lemma that 
constitutes a non-asymptotic bound for sampling.

\begin{lemma}
\label{lemma:subsample}
{[Lemma 10 of~\citet{lazygcn}]}
Let the subsampled function defined as 
\begin{equation*}
	F_{\mathcal{S}}(x) = \frac{1}{S} \sum_{k \in \mathcal{S}} F_k(x), 
	\E{F_{\mathcal{S}}(x)} = F(x), 
\end{equation*}
where $S = \left\vert \mathcal{S} \right\vert$ and 
$F_k(x) \in \mathbb{R}^d$ is $L$-Lipschitz continuous for all $k$. 
For $\eps \leq 2L$, we have with probability at least $1 - \delta$ that 
\begin{equation*}
	\norm{F_{\mathcal{S}}(x) - F(x)} \leq 
		4\sqrt{2}L \sqrt{\frac{\log(2d/\delta) + 1/4}{S}}.
\end{equation*}
\end{lemma}

\medskip

Based on Lemma~\ref{lemma:subsample}, we introduce 
the following lemma that provides an upper bound on 
the stochastic variance. 

\begin{lemma}
\label{lemma:variance}
Under Assumption~\ref{asp:gradient} 
and Assumption~\ref{asp:empirical}, 
for each step $t$, 
with probability at least $1 - \delta$, we have 
\begin{equation*}
	\enormsqr{\tg_t - \nabla f(\theta_t)} \leq 
	O\left(\frac{L^2 \log\left({2d}/{\delta}\right)}{B} 
	\left(1 + \frac{1}{W}\right) 
	+ \sigma^2 (2 - \rho)\right).
\end{equation*}
\end{lemma}

\begin{proof}
For simplicity, we omit the subscript of step $t$ in the proof. 
Since the inequality 
$\normsqr{x + y} \leq 2\normsqr{x} + 2\normsqr{y}$ 
holds, 
we only need to prove two claims. 

\begin{claim}
\label{claim:sampling}
Under the uniform sampling assumption 
and the Lipschitz assumption 
(Assumption~\ref{asp:gradient}(1) and  
Assumption~\ref{asp:empirical}(1)), 
with probability at least $1 - \delta$, we have  
\begin{equation*}
	\enormsqr{g - \nabla f(\theta)} \leq 
	O\left(\frac{L^2 \log\left({2d}/{\delta}\right)}{B} 
	\left(1 + \frac{1}{W}\right)\right).
\end{equation*}
\end{claim}

\begin{subproof}[Proof of Claim~\ref{claim:sampling}]
Since $g$ is computed over a mini-batch $\mathcal{B}$ 
subsampled from the workset $\mathcal{W}$, 
by denoting the gradients over $\mathcal{W}$ as 
$\bg = \frac{1}{WB} \sum_{k\in\mathcal{W}} \nabla f(\theta)$, 
we have
\begin{equation*}
	\mathbb{E}_{\mathcal{B}\sim\mathcal{W}}[g] = \bg, \;
	\mathbb{E}_{\mathcal{W}\sim\mathcal{D}}[\bg] = \nabla f(\theta).
\end{equation*}

Based on the $L$-Lipschitz assumption on $\nabla f$ 
and Lemma~\ref{lemma:subsample}, 
with probability at least $1 - \delta$, 
we have
\begin{equation*}
\begin{aligned}
\norm{g - \bg} 
&= \norm{\frac{1}{B} \sum_{k\in\mathcal{B}} \nabla f_k(\theta) 
	- \mathbb{E}_{\mathcal{B}\sim\mathcal{W}}\left[
		\frac{1}{B} \sum_{k\in\mathcal{B}} \nabla f_k(\theta)\right]} \\
&\leq 4\sqrt{2}L \sqrt{\frac{\log(2d/\delta) + 1/4}{B}}, 
\end{aligned}
\end{equation*}
and 
\begin{equation*}
\begin{aligned}
\norm{\bg - \nabla f(\theta)} 
&= \norm{\frac{1}{WB} \sum_{k\in\mathcal{W}} \nabla f_k(\theta) 
	- \mathbb{E}_{\mathcal{W}\sim\mathcal{D}}\left[
		\frac{1}{WB} \sum_{k\in\mathcal{W}} \nabla f_k(\theta)\right]} \\
&\leq 4\sqrt{2}L \sqrt{\frac{\log(2d/\delta) + 1/4}{WB}}, 
\end{aligned}
\end{equation*}

Denoting 
$\mathbb{E}_{\mathcal{W}\sim\mathcal{D}}\left[
	\mathbb{E}_{\mathcal{B}\sim\mathcal{W}} \left[\cdot\right]
\right]$ 
as $\mathbb{E}\left[\cdot\right]$ 
for simplicity, with probability at least $1 - \delta$, we have 
\begin{equation*}
\begin{aligned}
\mathbb{E}\left[\normsqr{g - \nabla f(\theta)}\right] 
&\leq 2\mathbb{E}\left[\normsqr{g - \bg}\right] 
	+ 2\mathbb{E}\left[\normsqr{\bg - \nabla f(\theta)}\right] \\
&\leq \frac{64L^2 \left(\log\left({2d}/{\delta}\right) + {1}/{4}\right)}{B}
	\left(1 + \frac{1}{W}\right).
\end{aligned}
\end{equation*}
Given that $2d/\delta \gg e^{1/4}$, we have 
\begin{equation*}
\mathbb{E}\left[\normsqr{g - \nabla f(\theta)}\right] 
\leq O\left(\frac{L^2 \log\left({2d}/{\delta}\right)}{B} 
	\left(1 + \frac{1}{W}\right)\right)
\end{equation*}
This completes the proof of Claim~\ref{claim:sampling}.
\end{subproof}

\begin{claim}
\label{claim:estimation}
Under the correct direction assumption 
and the bounded moment assumption 
(Assumption~\ref{asp:gradient}(2) and 
Assumption~\ref{asp:empirical}(2)), we have 
\begin{equation*}
	\enormsqr{\tg - g} \leq \sigma^2 (2 - \rho).
\end{equation*}
\end{claim}

\begin{subproof}[Proof of Claim~\ref{claim:estimation}]
The correct direction assumption implies 
\begin{equation*}
\begin{aligned}
\normsqr{\tg - g} 
&= \normsqr{\tg} + \normsqr{g} - 2\innerdot{\tg}{g} \\
&= \normsqr{\tg} + \normsqr{g} - 2\cos(\tg, g)\norm{\tg}\norm{g} \\
&\leq \normsqr{\tg} + \normsqr{g} - 2\rho\norm{\tg}\norm{g} \\
&= (1 - \rho)(\normsqr{\tg} + \normsqr{g}) 
	+ \rho(\norm{\tg} - \norm{g})^2.
\end{aligned}
\end{equation*}
Since the norms of gradients ($\norm{\tg}$ and $\norm{g}$) 
are non-negative, the inequality 
$(\norm{\tg} - \norm{g})^2 \leq 
\max\left\{\normsqr{\tg}, \normsqr{g}\right\}$ holds. 
Therefore, 
\begin{equation*}
\begin{aligned}
\normsqr{\tg - g} 
&\leq (1 - \rho)(\normsqr{\tg} + \normsqr{g}) 
	+ \rho(\norm{\tg} - \norm{g})^2 \\
&\leq 2(1 - \rho) \max\left\{\normsqr{\tg}, \normsqr{g}\right\} 
	+ \rho \max\left\{\normsqr{\tg}, \normsqr{g}\right\} \\
&= (2 - \rho) \max\left\{\normsqr{\tg}, \normsqr{g}\right\}. 
\end{aligned}
\end{equation*}
Combining with the bounded moment assumption, we have 
\begin{equation*}
\enormsqr{\tg - g} \leq \sigma^2 (2 - \rho).
\end{equation*}
This completes the proof of Claim~\ref{claim:estimation}.
\end{subproof}

Finally, combining Claim~\ref{claim:sampling}, 
Claim~\ref{claim:estimation} and the inequality that 
$\normsqr{x + y} \leq 2\normsqr{x} + 2\normsqr{y}$, 
we complete the proof of Lemma~\ref{lemma:variance}. 
\end{proof}

\medskip

Next, we present the proof for Theorem~\ref{thm:convergence}.

\setcounter{theorem}{0}
\begin{theorem}
\label{proof:convergence}
Under Assumption~\ref{asp:gradient},\ref{asp:empirical} 
and setting the step size appropriately, 
after $T$ steps, select $\bar{\theta}_{T}$ randomly from 
$\{\theta_0, \theta_1, \dots, \theta_{T-1}\}$. 
With probability at least $1 - \delta$, we have 
\begin{equation*}
	\mathbb{E}\left[\lVert{\nabla f(\bar{\theta}_T)}\rVert^2\right]
    \leq O\left(\sqrt{{\left(\frac{L^2 \log\left({2d}/{\delta}\right)}{B} 
	\left(1 + \frac{1}{W}\right) 
	+ \sigma^2 (2 - \rho)\right)}/{T}}\right).
\end{equation*}
\end{theorem}

\begin{proof}
By Taylor's Expansion Formula with Lagrangian Remainder, 
we have 
\begin{equation*}
\begin{aligned}
\label{eq:erm}
f(\theta_{t+1}) 
&= f(\theta_t - \eta \tg_t) \\
&\leq f(\theta_t) - \eta \innerdot{\tg_t}{\nabla f(\theta_t)} 
	+ \frac{L}{2} \eta^2 \normsqr{\tg_t} \\
&= f(\theta_t) - \eta \innerdot{\nabla f(\theta_t) + 
			\tg_t - \nabla f(\theta_t)}{\nabla f(\theta_t)}
	+ \frac{L}{2} \eta^2 \normsqr{\tg_t} \\
&= f(\theta_t) - \eta \normsqr{\nabla f(\theta_t)} 
	- \eta \innerdot{\tg_t - \nabla f(\theta_t)}{\nabla f(\theta_t)} \\
	&\;\;\;+ \frac{L}{2} \eta^2 \normsqr{\tg_t},
\end{aligned}
\end{equation*}
where the first inequality is due to 
the property of Lipschitz continuity. 
By substituting $\normsqr{\tg_t}$ as 
\begin{equation*}
\begin{aligned}
\normsqr{\tg_t} 
&= \normsqr{\tg_t - \nabla f(\theta_t) + \nabla f(\theta_t)} \\
&= \normsqr{\tg_t - \nabla f(\theta_t)} + \normsqr{\nabla f(\theta_t)} 
	+ 2\innerdot{\tg_t - \nabla f(\theta_t)}{\nabla f(\theta_t)}, 
\end{aligned}
\end{equation*}
we have 
\begin{equation*}
\begin{aligned}
f(\theta_{t+1}) 
&\leq f(\theta_t) 
	+ (L\eta^2 - \frac{3}{2}\eta) 
		\normsqr{\nabla f(\theta_t)} \\
	&\;\;\;+ (L\eta^2 - \frac{1}{2}\eta) 
		\normsqr{\tg_t - \nabla f(\theta_t)} \\
&\leq f(\theta_t) 
	+ (L\eta^2 - \frac{3}{2}\eta) 
		\normsqr{\nabla f(\theta_t)} \\
	&\;\;\;+ L\eta^2 \normsqr{\tg_t - \nabla f(\theta_t)}.
\end{aligned}
\end{equation*}
Taking expectation on both sides, we have 
\begin{equation*}
\begin{aligned}
(\frac{3}{2}\eta - L\eta^2 ) \enormsqr{\nabla f(\theta_t)}
&\leq \E{f(\theta_t)} - \E{f(\theta_{t+1})} \\
	&\;\;\;+ L\eta^2 \enormsqr{\tg_t - \nabla f(\theta_t)}
\end{aligned}
\end{equation*}
Summing over $T$ steps and dividing both sides by 
$(\frac{3}{2}\eta - \eta^2 L) T$, 
and noticing that $\bar{\theta}_{T}$ is selected randomly 
from $\{\theta_0, \theta_1, \dots, \theta_{T - 1}\}$, we have 
\begin{equation*}
\begin{aligned}
\enormsqr{\nabla f(\bar{\theta}_T)} 
&= \frac{1}{T} \sum_{t=0}^{T-1} \enormsqr{\nabla f(\theta_t)} \\
&\leq \frac{2}{(3\eta - L2\eta^2) T} 
	(f(\theta_0) - \mathbb{E}[f(\theta_t)]) \\
& \;\;\;+ \frac{2L\eta}{(3 - 2L\eta) T}
	\sum_{t=0}^{T-1} \enormsqr{\tg_t - \nabla f(\theta_t)} \\
&\leq \frac{2}{(3\eta - L2\eta^2) T} (f(\theta_0) - f(\theta^*)) \\
& \;\;\;+ \frac{2L\eta}{(3 - 2L\eta) T}
	\sum_{t=0}^{T-1} \enormsqr{\tg_t - \nabla f(\theta_t)} \\
\end{aligned}
\end{equation*}

Denote $V = \frac{1}{T} \sum_{t=0}^{T-1} 
	\enormsqr{\tg_t - \nabla f(\theta_t)}$. 
According to Lemma~\ref{lemma:variance}, 
$V \leq O\left(\frac{L^2 \log\left({2d}/{\delta}\right)}{B} 
	\left(1 + \frac{1}{W}\right) 
	+ \sigma^2 (2 - \rho)\right)$ 
satisfies with probability at least $1 - \delta$. 
Therefore, letting $\eta = \min\left\{ 
\frac{3}{2L}, \sqrt{\frac{f(\theta_0) - f(\theta^*)}{LTV}} 
\right\}$, we have
\begin{equation*}
\begin{aligned}
\enormsqr{\nabla f(\bar{\theta}_T)} 
&\leq \frac{2(f(\theta_0) - f(\theta^*))}{(3\eta - L2\eta^2) T}  
	+ \frac{2L\eta V}{(3 - 2L\eta)} 
\leq O\left(\sqrt{\frac{V}{T}}\right) \\
&\leq O\left(\sqrt{\frac{\frac{L^2 \log\left({2d}/{\delta}\right)}{B} 
	\left(1 + \frac{1}{W}\right) 
	+ \sigma^2 (2 - \rho)}{T}}\right). 
\end{aligned}
\end{equation*}

\end{proof}

\fi
%%%%%%%%%%%%%%%%%%%%%%%%%%%%

\end{document}